\documentclass[runningheads]{llncs}
\usepackage{graphicx}
\usepackage{amsmath}
\usepackage{amssymb}
\usepackage{algorithm}
\usepackage{algorithmic}
\usepackage{color}
\usepackage{url}
\usepackage{booktabs}
\newcommand{\hsqt}[0]{\frac{\sqrt{3}}{2}}
\newcommand{\half}[0]{\frac{1}{2}}

\makeatletter
\AtBeginDocument{%
  \def\doi#1{\url{https://doi.org/#1}}}
\makeatother

\begin{document}
\title{Generalizing and Unifying Gray-box Combinatorial Optimization Operators}
%
%
\author{Francisco Chicano\inst{1}\orcidID{0000-0003-1259-2990} \and
Darrell Whitley\inst{2}\orcidID{0000-0002-2752-6534} \and
Gabriela Ochoa\inst{3}\orcidID{0000-0001-7649-5669} \and
Renato Tinós\inst{4}\orcidID{0000-0003-4027-8851}
}
\authorrunning{Chicano, Whitley, Ochoa, Tinós}
%
\institute{ITIS Software, University of Malaga, Spain \\
\email{chicano@uma.es} \and
Colorado State University, USA \\
\email{whitley@cs.colostate.edu} \and
University of Stirling, Scotland, UK \\
\email{gabriela.ochoa@stir.ac.uk} \and
University of São Paulo, Brazil \\
\email{rtinos@ffclrp.usp.br}}
\maketitle              
\begin{abstract}
Gray-box optimization leverages the information available about the mathematical structure of an optimization problem to design efficient search operators.  Efficient hill climbers and crossover operators have been proposed in the domain of pseudo-Boolean optimization and also in some permutation problems. However, there is no general rule on how to design these efficient operators in different representation domains. This paper proposes a general framework that encompasses all known gray-box operators for combinatorial optimization problems. The framework is general enough to shed light on the design of new efficient operators for new problems and representation domains. We also unify the proofs of efficiency for gray-box hill climbers and crossovers and show that the mathematical property explaining the
speed-up of gray-box crossover operators, also explains the efficient identification of improving moves in gray-box hill climbers. We illustrate the power of the new framework by proposing an efficient hill climber and crossover for two related permutation problems: the Linear Ordering Problem and the Single Machine Total Weighted Tardiness Problem.
\end{abstract}
\section{Introduction}

In the context of combinatorial optimization with metaheuristics, it is common to assume that the algorithm does not know anything about the optimization problem except the quality of the solutions. These values guide the search towards better solutions. The objective function $f$ is a \emph{black-box} that takes a solution $x$ and returns a real value $f(x)$ representing its cost (to minimize) or quality (to maximize). For real-world applications, we often have more information than just a black-box evaluation function.  During the last 15 years, several proposals used this additional problem information to design more effective variation operators. 
This is how operators like Partition Crossover for the Traveling Salesman Problem (TSP)~\cite{WhitleyHH09} or the constant time hill climber for NK landscapes and MAX-$k$-SAT~\cite{WhitleyC12} appeared. This line of research, known as \emph{gray-box optimization}~\cite{WhitleyCG16}, has produced very efficient algorithms like DRILS~\cite{ChicanoWOT17} and PRILS~\cite{CanonneDCO23}, and very efficient ways to improve existing algorithms~\cite{ChenWTC18}.  Many so-called ``black-box" optimization methods are not black-box optimizers;  every modern inexact MAX-SAT or TSP solver exploits problem structure \cite{helsgaun2000effective,Hoos2004SLS,Hoos2004maxsat,nagata2013}.

In the gray-box optimization literature, we can find examples of hill climbers that can identify improving moves in constant time and crossover operators able to explore an exponential number of offspring in linear time. We prove in this paper that the key ingredient behind the efficiency in both cases is the same mathematical identity (decomposition theorem). We propose in Section~\ref{sec:framework} a mathematical framework that unifies the math behind the efficient gray-box crossovers and hill climbers. We also show that the results are general and can be applied to any representation space. To illustrate how to use the framework for designing new gray-box operators, we propose in Section~\ref{sec:new-gb-operators} a new hill climber and partition crossover based on the same mathematical identity for the Linear Ordering Problem and the Single Machine Total Weighted Tardiness, two NP-hard permutation problems.

The new framework allows us to revisit existing gray-box operators (see Section~\ref{sec:previous-gb-operators}) to find opportunities for improvement that were not obvious in the original operator formulation. These opportunities appear after analyzing the framework using the Fourier transform over finite groups. We include in Section~\ref{sec:background} a short introduction to the concepts of group, group representation, and Fourier transform over finite groups to make the paper as self-contained as possible.

\section{Background}
\label{sec:background}


A \emph{group} is a pair $(G, \cdot)$, where $G$ is a set of elements and $\cdot: G \times G \rightarrow G$ is a binary operator defined over the elements of the group with the following properties:
\begin{itemize}
    \item Associative: $g_1 \cdot (g_2 \cdot g_3)=(g_1 \cdot g_2) \cdot g_3$
    \item Neutral element: there is an element $e\in G$ such that $e \cdot g = g \cdot e = g$ for all $g \in G$.
    \item Inverse: for each $g \in G$ there is another element $g^{-1}$ such that $g \cdot g^{-1} = g^{-1} \cdot g = e$.
\end{itemize}

Two examples of finite groups relevant in this paper are $(\mathbb{Z}_2^n,+)$, that is, binary strings with the bitwise exclusive OR (XOR); and the symmetric group (permutations) $(S_n,\circ)$ with the composition operator. When the operation is clear from the context, we will omit it and simply write $\mathbb{Z}_2^n$ or $S_n$. Observe that while in the binary case, the operation (the sum) is commutative, in the permutation case the composition is not commutative. A \emph{group action} of a group $G$ over a set $X$ is a homomorphism of the group $G$ into a subgroup of the bijective functions of $X$ to itself. That is, if $g \in G$, $x \in X$ and $g x \in X$ represents the action of $g$ applied to element $x$, then $(g \cdot h)x = g (h x)$.

\subsection{Group representation}

A \emph{representation} of a group $G$ is a mapping $\rho:G \rightarrow GL(V)$ between the elements of the group $G$ and the automorphisms of a vector space $V$ such that $\rho(g_1 \cdot g_2)=\rho(g_1) \circ \rho(g_2)$. 
An \emph{automorphism} is an invertible linear map from a vector space $V$ to itself.
Without loss of generality, we will assume in the following $V=\mathbb{C}^n$, where the vectors are $n$-tuples of complex numbers and the automorphisms are non-singular squared complex matrices of size $n \times n$, that is, matrices with a nonzero determinant. 
A group representation translates the group elements into matrices and the group operation into matrix multiplication. 

Not all group representations are equally important.  Two representations, $\rho_1$ and $\rho_2$ are \emph{equivalent} if there exists a matrix $P\in \mathbb{C}^{n\times n}$ such that $\rho_1(g) = P \rho_2(g) P^{-1}$ for all $g \in G$. Equivalent representations provide the same information; we will be interested in a set of \emph{inequivalent} representations. A representation $\rho$ is \emph{reducible} if we can write $\rho(g)$ as a block-diagonal matrix
\begin{equation}
    \rho(g) = \left( \begin{array}{c|c}
    \rho_1(g) & 0 \\
    \hline
    0 & \rho_2(g)
    \end{array}\right),
\end{equation}
for all $g\in G$. Reducible representations do not provide more information than their component representations and we will discard them, focusing on \emph{irreducible} representations. When the group is finite there is also a finite set of \emph{inequivalent irreducible} representations, called \emph{irreps}, and the cardinality of this set is exactly the number of conjugacy classes of the group~\cite{Fulton2004}.
When the group is commutative all the irreps are scalar values (matrices of size $1 \times 1$). This is the case in the group $\mathbb{Z}_2^{n}$, the one we use for pseudo-Boolean optimization. The irreps of $\mathbb{Z}_2^{n}$ are the Walsh functions~\cite{Walsh:1923}. However, when the group is not commutative, some irreps need to be expressed with squared matrices of size at least 2, as the next example shows.

\begin{example}
We will denote the permutations using the \emph{cycle notation}. A cycle $(c_1\; c_2\; \ldots\; c_k)$ represents a permutation that maps $c_1$ into $c_2$, $c_2$ into $c_3$ and so on. Element $c_k$ is mapped into $c_1$. The elements missing in the cycle are mapped into themselves.
For example, $(1 \; 2)$ denotes a permutation in which element 1 is mapped into element 2 and element 2 is mapped into element 1. The identity permutation, which maps every element to itself, is commonly denoted with $(1)$. Every arbitrary permutation can be written as the product of disjoint cycles.

Let's consider the group $S_3$ of permutations of three elements. A representation $\rho$ for their elements is in Table~\ref{tab:s3-repr}.
\begin{table}[!ht]
\centering
\caption{Example of representation of $S_3$ using $2 \times 2$ matrices. This is the $\rho_{(2,1)}$ irreducible representation in the Young Orthogonal Representation.}
\label{tab:s3-repr}
\begin{tabular}{ccccccc}
\toprule
$\pi$ & (1) & (2 3)  & (1 2) & (1 2 3) & (1 3 2) & (1 3) \\
\hline
$\rho(\pi)$ &
$\left(\begin{array}{rr}
1 & 0 \\
0 & 1
\end{array}\right)$ &
$\left(\begin{array}{rr}
-\half & \hsqt \\
\hsqt & \half
\end{array}\right)$ &
$\left(\begin{array}{rr}
1 & 0 \\
0 & -1
\end{array}\right)$ &
$\left(\begin{array}{rr}
-\half & -\hsqt \\
\hsqt & -\half
\end{array}\right)$ &
$\left(\begin{array}{rr}
-\half & \hsqt \\
-\hsqt & -\half
\end{array}\right)$ &
$\left(\begin{array}{rr}
-\half & -\hsqt \\
-\hsqt & \half
\end{array}\right)$ \\
\bottomrule
\end{tabular}
\end{table}
\qed
\end{example}

\subsection{Fourier transform over finite groups}

Let $f: G \rightarrow \mathbb{C}$ be a complex-valued function defined on group $G$ and $\rho$ a representation of $G$. The \emph{Fourier transform} of $f$ at $\rho$ is defined as
\begin{equation}
\label{eqn:fourier-transform}
\hat{f}(\rho) = \sum_{g \in G} f(g) \rho(g).
\end{equation}

Observe that $\hat{f}(\rho)$ is, in general, a matrix because $\rho(g)$ is a matrix. In general, the Fourier transform of a function $f$ at a particular representation does not provide all the information to reconstruct $f$,  
but the Fourier transform at a set of irreps does. The inverse Fourier transform is defined as
\begin{equation}
\label{eqn:inverse-fourier}
f(g) = \frac{1}{|G|} \sum_{\rho \in \text{irreps}} d_\rho \text{Tr} \left( \hat{f}(\rho) \rho(g)^{-1}  \right),
\end{equation}
where $d_{\rho}$ is the dimension of the vector space associated to representation $\rho$, $\text{Tr}$ is the trace of a matrix, and $\rho(g)^{-1}$ denotes the inverse  of $\rho(g)$.


\section{The unifying framework}
\label{sec:framework}

In this section, we present a mathematical framework that summarizes in one single theorem the mathematical background behind some efficient hill climbers and crossovers in gray-box optimization. In particular, the framework generalizes the math behind Hamming Ball Hill Climber (HBHC)~\cite{ChicanoWS14} and Partition Crossover (PX)~\cite{TinosWC15,TinosWO14,WhitleyHH09}. Explaining with a formula the math behind hill climbers and crossovers yields a unification in gray-box combinatorial optimization. Moreover, the formula does not assume a concrete representation, but it is valid for any search space. This yields a generalization of all gray-box results known up to date. The new framework should be able to open the door to new efficient gray-box operators in combinatorial optimization problems
We illustrate this generalization by providing two new gray-box operators in Section~\ref{sec:new-gb-operators}.

We assume that our search space is a finite set $X$. 
A \emph{move} in the solution space is a function $h: X \rightarrow X$ that maps any solution $x \in X$ to another solution $h(x)$ that is generated by that move.
\begin{example}

One example of a move is a bit flip in the binary representation. If $1001 \in \mathbb{Z}_2^4$ is a binary string representing a solution of an optimization problem, we can flip the first bit to get $0001$. Here, the move is represented by the function $h(x)=x + 1000$ in $\mathbb{Z}_2^4$. \qed
\end{example}
\begin{example}
One example of a move in the permutation space is an exchange of two positions in the permutation. The move $h(\sigma)= \sigma \cdot (3 \; 7)$ represents an exchange of elements 3 and 7 (we use cycle notation), and the move $h(\sigma)= (3 \; 7) \cdot \sigma $ represents an exchange of elements \emph{at positions} 3 and 7. \qed
\end{example}

We now define the \emph{delta function} for a move $h$ and function $f: X \rightarrow \mathbb{R}$ as
\begin{align}
& \label{eqn:delta-r} \left(\Delta_{h} f\right)(x) = f(h(x)) - f(x).
\end{align}


We  denote the composition of two moves $h_1$ and $h_2$  with $h_2 \circ h_1$ and we define it as $(h_2 \circ h_1) (x) = h_2(h_1(x))$. This is equivalent to applying move $h_1$ first to $x$ and then applying move $h_2$. Composition is not commutative in general, but we will focus on commutative moves in this paper.

\begin{definition}
Let $h_1, h_2: X \rightarrow X$ be two moves that commute under function composition, that is, $h_1 \circ h_2 = h_2 \circ h_1$. We say that the moves are \emph{non-interacting for the set of solutions $Y \subseteq X$} when
\begin{equation}
\label{eqn:non-interaction-def}\left(\Delta_{h_2} f\right)(h_1(x)) = \left(\Delta_{h_2} f\right)(x) \;\; \forall x \in Y.
\end{equation}
If we omit the reference to the set $Y$, then the moves are non-interacting for the whole search space $X$.
\end{definition}

We can characterize the non-interaction using a different but equivalent expression, which is more common in the literature.
\begin{proposition}
Let $h_1, h_2: X \rightarrow X$ be two moves that commute, then the moves are \emph{non-interacting} for set $Y \subseteq X$ if and only if
\begin{equation}
    \label{eqn:non-interaction-char}\left(\Delta_{h_1 \circ h_2} f\right)(x) = \left(\Delta_{h_1} f\right)(x) + \left(\Delta_{h_2} f\right)(x) \;\; \forall x \in Y.
\end{equation}
\end{proposition}
\begin{proof}
    Let's transform the equality~\eqref{eqn:non-interaction-char} into the equality~\eqref{eqn:non-interaction-def} using definitions and algebraic manipulations:
    \begin{align*}
\left(\Delta_{h_1 \circ h_2} f\right)(x) &= \left(\Delta_{h_1} f\right)(x) + \left(\Delta_{h_2} f\right)(x)\\
f(h_1(h_2(x))) - f(x) &= f(h_1(x)) - f(x) + f(h_2(x)) - f(x) \text{ (definition of $\Delta$)} \\
f(h_1(h_2(x))) - f(h_2(x))  &= f(h_1(x)) - f(x) \text{ (reordering terms)}\\
\left(\Delta_{h_1} f\right)(h_2(x)) &= \left(\Delta_{h_1} f\right)(x) \text{ (definition of $\Delta$)}.
    \end{align*}
\qed
\end{proof}

The efficiency and efficacy of the gray-box operators are based on the separability of the delta function when we decompose a move. We prove here a generalization of this separability that is valid for all representations.

\begin{theorem}[Decomposition]
\label{thm:decomposition}
    Let $h_i: X \rightarrow X$ with $1\leq i \leq m$ be a family of pairwise commutative moves and the set $Y \subseteq X$. For each pair of moves $h_i$, $h_j$ we also require that they are non-interacting in the sets $H(Y)$ for all moves $H$ that are compositions of moves $h_k$ with $1\leq k \leq m$ that do not include $h_i$ and $h_j$. $H$ can also be the identity function (empty composition of moves). It holds
\begin{equation}
    \label{eqn:delta-sum} \left(\Delta_{\bigcirc_{i=m}^{1} h_i} f\right)(x) = \sum_{i=1}^{m} \left(\Delta_{h_i} f\right)(x) \;\; \forall x \in Y.
\end{equation}
\end{theorem}
\begin{proof}
    We can prove Eq.~\eqref{eqn:delta-sum} using induction over $m$. For $m=1$ the equality is trivial because we have the same expression on both sides. Let's assume that the equality holds for $m$ and let's prove it for $m+1$. We assume $x \in Y$ and write
    \begin{align*}
        \left(\Delta_{\bigcirc_{i=m+1}^{1}  h_i} f\right)(x) &= f((\bigcirc_{i=m+1}^{1}  h_i)(x)) - f(x) \\
        &= f((\bigcirc_{i=m+1}^{1}  h_i)(x)) - f((\bigcirc_{i=m}^{1}  h_i)(x)) \\
        & \qquad+ f((\bigcirc_{i=m}^{1}  h_i)(x))- f(x) \\
        &= (\Delta_{h_{m+1}} f) ((\bigcirc_{i=m}^{1}  h_i)(x)) + \left(\Delta_{\bigcirc_{i=m}^{1} h_i} f\right)(x).
    \end{align*}
    Since $h_{m+1}$ and all the $h_{i}$ are commutative and non-interacting in $H(Y)$ for any composition $H$ of the other moves we can apply Eq.~\eqref{eqn:non-interaction-def} iteratively to get $(\Delta_{h_{m+1}} f) ((\bigcirc_{i=m}^{1}  h_i)(x)) = (\Delta_{h_{m+1}} f) (x) $. If we use the induction hypothesis we can write
    \[
    (\Delta_{h_{m+1}} f) ((\bigcirc_{i=m}^{1}  h_i)(x)) + \left(\Delta_{\bigcirc_{i=m}^{1} h_i} f\right)(x) = (\Delta_{h_{m+1}} f) (x) +  \sum_{i=1}^{m} \left(\Delta_{h_i} f\right)(x),
    \]
    proving the result. \qed
\end{proof}

\begin{corollary}[Lattice linear equation]
    \label{cor:lattice}
    Let $W \subseteq \{1,2,\ldots,m\}$ be a subset of indices for the moves $h_i$. We denote with $H_W$ the move created after the composition of the $h_i$ with $i \in W$, $H_W=\bigcirc_{i \in W} h_i$. Then,
    \begin{equation}
    \label{eqn:lattice}
    \left(\Delta_{H_W} f\right)(x) = \sum_{i \in W} \left(\Delta_{h_i} f\right)(x) \;\; \forall x \in Y.
    \end{equation}
\end{corollary}
\begin{proof}
    It is a simple application of Theorem~\ref{thm:decomposition} to the moves $h_i$ with $i \in W$.
\end{proof}

Let's now focus on the moves that are bijections. These moves form a group, the symmetric group (permutations) over set $X$, denoted here with $\mathcal{S}_X$. This group naturally \emph{acts} on the search space $X$. That is, we define the action of a bijective move $h$ over an element $x \in X$ as $h(x)$. We will be interested in this paper in a subgroup $G$ of $\mathcal{S}_X$ that is \emph{transitive}, that is, for any pair of elements $x,y \in X$ we can find a move $g \in G$ such that $g(x)=y$. If we work with such a group $G$, then we can fix one solution in the search space $x_0\in X$ as a \emph{base solution} and define a group function $f': G \rightarrow \mathbb{R}$ based on the objective function $f$ as $f'(g) = f(g(x_0))$. Since the group is transitive, any solution $x$ of the search space has at least one move $g$ such that $g(x_0)=x$. Optimizing $f'$ is equivalent to optimizing $f$, and now the domain of $f'$ is a group. We can represent both, solutions and moves, with elements of that group. 

\begin{example}
    In the context of pseudo-Boolean optimization, where the solutions are elements of $\mathbb{Z}_2^n$ (binary strings), we can also use binary strings to represent a move. The element $w \in \mathbb{Z}_2^n$ acts on solution $x$ by moving it to $x + w$ (bitwise XOR).
    In the permutation space, a move can also be a permutation $\sigma$ that acts on solution $\pi$ by left-composition, $\sigma \circ \pi$, or right-composition, $\pi \circ \sigma$. \qed
\end{example}

Based on the decomposition theorem (Theorem \ref{thm:decomposition}) we can propose a method to potentially develop an efficient hill climber and a general partition crossover operator. 

\subsection{Efficient hill climber}
\label{subsec:hill-climber}

The idea of the efficient hill climber (Algorithm~\ref{alg:hill-climber}) is to keep in memory the values $(\Delta_{h}f)(x)$ for the current solution $x$ of a set of moves $h \in N$. With a little abuse of notation, we will denote the vector of delta values, also called \emph{score vector},  with $\Delta f$.
Moves are classified into two classes: improving and non-improving\footnote{Alternatively, it is possible to use multiple classes based on the level of improvement.}. Assuming minimization, a move $h$ is improving if and only if $(\Delta_{h}f)(x) < 0$. At one iteration, the hill climber picks one of the improving moves $h$ and updates the current solution by applying $h$ to $x$ (see Line~\ref{lin:move} in Algorithm~\ref{alg:hill-climber}). It also updates the vector $(\Delta f)$ (Line~\ref{lin:update}).

\begin{algorithm}[!ht]
\begin{algorithmic}[1]
\WHILE{$(\Delta_{h}f)(x) < 0$ for some $h\in N$}
\label{lin:innerloop}
	\STATE $h \leftarrow $ selectImprovingMove($\Delta f$);
	\label{lin:select-move}
	\STATE updateDeltas($\Delta f$,$x$,$h$);
	\label{lin:update}
	\STATE $x \leftarrow h(x)$;
	\label{lin:move}
\ENDWHILE
\end{algorithmic}
\caption{Efficient hill climber}
\label{alg:hill-climber}
\end{algorithm}

The selection of an improving move can be done in constant time. In many cases, the update of $\Delta f$ can also be done in constant time if the set of moves $N$ is linear in the size of the problem. This makes it possible to run one step of the hill climber in constant time.
Now, we claim that there is no need to keep in memory the value $(\Delta_{h_2 \circ h_1}f)(x)$ if $h_1$ and $h_2$ are commutative and non-interacting. For the move $h_2 \circ h_1$ to be improving, we need at least one of $h_1$ or $h_2$ to be an improving move. Certainly, if $(\Delta_{h_2 \circ h_1}f)(x) = (\Delta_{h_1}f)(x) + (\Delta_{h_2}f)(x) < 0$ then the two values $(\Delta_{h_1}f)(x)$ and $(\Delta_{h_2}f)(x)$ cannot be greater or equal to zero at the same time. This result generalizes for any number of moves $h_i$ that are commutative and non-interacting. That means that we can explore a large neighborhood $N(x)$ around a solution by analyzing a much smaller number of moves. The key point here is to find the set of moves needed to explore the desired neighborhood. In the context of pseudo-Boolean optimization, it has been proven that for $k$-bounded pseudo-Boolean functions where each variable can only appear in a number of subfunctions, it is possible to explore a ball of constant size $t$ storing only a linear number of moves in memory~\cite{ChicanoWS14}.

\subsection{Partition crossover}
\label{subsec:partition-crossover}

Corollary~\ref{cor:lattice} also makes it possible to define a very effective crossover operator, called partition crossover (PX).
Let $x_1, x_2 \in X$ be two solutions that are connected by a move $g$, that is, $g(x_1)=x_2$. Partition crossover is based on the decomposition of $\Delta_{g} f$ for solution $x_1$. Let us write this decomposition as
\begin{equation}
\left(\Delta_{g} f\right)(x_1)  = \sum_{i=1}^{q} \left(\Delta_{h_i} f\right)(x_1),
\end{equation}
then each $h_i$ is called a \emph{component} in partition crossover. For each component, we decide if it is included or not in the offspring solution. We can write this more explicitly evaluating the delta function in $x_1$:
\begin{equation}
\left(\Delta_{g} f\right) (x_1) = f (x_2) - f(x_1) = \sum_{i=1}^{q} \left(f(h_i (x_1)) - f(x_1))\right) .
\end{equation}
If $(\Delta_{h_i} f)(x_1) < 0$ then $h_i$ is used to form the offspring solution, otherwise it is not included. We can now define the set $W$ of the corollary as follows: $W=\{i \,|\, 1\leq i\leq m, (\Delta_{h_i} f)(x_1) < 0\}$. The offspring is the solution $x^* = H_W(x_1)$. 
According to Corollary~\ref{cor:lattice}, this procedure necessarily finds the best solution from the set of $2^m$ solutions that can be generated by applying or not each of the $m$ components.

\subsection{The role of Fourier transforms}
\label{subsec:fourier-decomp}

Let us consider the case in which the moves are bijections and we can ground the objective function on a group $G$. We can use the Fourier transform of the objective function to help identify non-interacting moves to apply  Theorem~\ref{thm:decomposition}.

\begin{theorem}
\label{thm:decomposition-fourier}
    Let $G$ be a group, $f: G \rightarrow \mathbb{R}$ a function and $h_1, h_2 \in G$ two moves that commute. Then, the moves are non-interacting if and only if:
\begin{equation}
\label{eqn:decomposition-fourier}
\left(\rho_{\lambda}(h_1^{-1})-I\right)\left(\rho_{\lambda}(h_2^{-1})-I\right) \hat{f}(\lambda) = 0 \;\;\; \forall \lambda \in \text{irreps},
\end{equation}
where $\rho_{\lambda}$ is a representation of $G$ labeled with $\lambda$ and $\hat{f}(\lambda)$ is the Fourier transform of $f$ in that representation.
\end{theorem}
\begin{proof}
Let us apply Eq.~\eqref{eqn:fourier-transform} (Fourier transform) to $\Delta_h f$:
\begin{align}
    \nonumber \sum_{g \in G} \left( \Delta_h f \right) (g) \rho_{\lambda}(g) &= \sum_{g \in G} \left(f(h \cdot g) - f(g) \right) \rho_{\lambda} (g) \\
    \nonumber &= \left( \sum_{g \in G} f(h \cdot g) \rho_{\lambda} (g) \right)  - \left( \sum_{g \in G} f(g) \rho_{\lambda} (g) \right) \\
    \nonumber \intertext{now using the change  $g'=h \cdot g$}
    \nonumber &= \left( \sum_{g' \in G} f(g') \rho_{\lambda} (h^{-1} \cdot g') \right)  - \left( \sum_{g \in G} f(g) \rho_{\lambda} (g) \right)\\
    \intertext{applying $\rho_{\lambda}(h^{-1} \cdot g')=\rho_{\lambda}(h^{-1}) \rho_{\lambda}(g')$} 
    \nonumber &= \rho_{\lambda} (h^{-1}) \left( \sum_{g' \in G} f(g') \rho_{\lambda} (g') \right)  - \left( \sum_{g \in G} f(g) \rho_{\lambda} (g) \right) \\
    \label{eqn:fourier-delta} &= \rho_{\lambda} (h^{-1}) \hat{f}(\lambda)  - \hat{f}(\lambda) = \left(\rho_{\lambda} (h^{-1})-I\right)\hat{f}(\lambda).
\end{align}

Let us now write Eq.~\eqref{eqn:non-interaction-char} as
\begin{equation}
\nonumber \Delta_{h_2 \cdot h_1} f - \Delta_{h_2} f - \Delta_{h_1} f = 0.
\end{equation}
If we apply Eq.~\eqref{eqn:fourier-delta} to the previous expression we get for all irreps $\lambda$
\begin{align*}
    0 &=\left[\rho_{\lambda} \left((h_2\cdot h_1)^{-1}\right)-I\right]\hat{f}(\lambda)
    -\left[\rho_{\lambda} (h_2^{-1})-I\right]\hat{f}(\lambda)
    -\left[\rho_{\lambda} (h_1^{-1})-I\right]\hat{f}(\lambda)
    \\
    &=\left[\rho_{\lambda} \left((h_2\cdot h_1)^{-1}\right)-I - \rho_{\lambda} (h_2^{-1})+I - \rho_{\lambda} (h_1^{-1})+I \right] \hat{f}(\lambda) \\
    &=\left[\rho_{\lambda} \left((h_2\cdot h_1)^{-1}\right) - \rho_{\lambda} (h_2^{-1}) - \rho_{\lambda} (h_1^{-1})+I \right] \hat{f}(\lambda) \\ 
    &=\left[\rho(h_1^{-1})  \rho(h_2^{-1}) - \rho_{\lambda} (h_2^{-1}) - \rho_{\lambda} (h_1^{-1})+I \right] \hat{f}(\lambda) \\
    &=\left(\rho_{\lambda}(h_1^{-1})-I\right)\left(\rho_{\lambda}(h_2^{-1})-I\right) \hat{f}(\lambda).
\end{align*}
\qed
\end{proof}

Observe that $\rho_\lambda$ is, in general, a matrix, as well as $\hat{f}(\lambda)$. Having the Fourier transform of $f$ it is possible to find pairs of commutative moves $h_1$ and $h_2$ that are non-interacting and can be used to produce a hill climber or partition crossover tailored to the particular instance of the problem.

\begin{example}
Since $\mathbb{Z}_2^{n}$ is commutative, each solution is a conjugacy class and the irreps are scalar functions labeled with the elements of $\mathbb{Z}_2^{n}$: the well-known Walsh functions, defined as $\varphi_{\lambda}(x)=\prod_{i=1}^{n}(-1)^{\lambda_i x_i} $, which take values $1$ and $-1$.  The moves $h_1$ and $h_2$ are also elements of $\mathbb{Z}_2^{n}$ (binary strings). We have for this case $h^{-1}=h$ (each element is its own inverse element) and $\rho_{\lambda}(x^{-1})=\rho_{\lambda}(x)=\varphi_{\lambda}(x)$. In this context, Eq.~\eqref{eqn:decomposition-fourier} is satisfied, and two moves $h_1$ and $h_2$ are non-interacting if for all $\lambda$ for which there is a nonzero Walsh coefficient in the Fourier transform of $f$, $\hat{f}(\lambda) \neq 0$, we have that either $\varphi_{\lambda}(h_1)=1$ or $\varphi_{\lambda}(h_2)=1$. \qed
\end{example}

\section{Previous gray-box operators}
\label{sec:previous-gb-operators}

In this section, we will prove that some gray-box operators previously published in the literature fit in the framework.
In the binary case, the new framework provides new opportunities to improve the operators.

\subsection{Partition crossover in the TSP}
\label{subsec:gapx-tsp}

Let us focus on the Generalized Asymmetric Partition Crossover (GAPX) by Tinós et al.~\cite{TinosWO14}, which was defined for the asymmetric TSP. The basic move we will need in this case is an \emph{insertion}. We will denote with $h_{i,j}$ with $i\neq j$ the insertion of city $i$ before city $j$ in the solution. 
\begin{example}
In Figure~\ref{fig:gapx-tsp} we can see that city $2$ is removed from its position and placed before city $20$ when we move from Solution 1 to Solution 2. We represent that move with $h_{2,20}$. Two insertion moves $h_{i,j}$ and $h_{i',j'}$ are commutative if the sets $\{i,j\}$ and $\{i',j'\}$ are disjoint. That is, if they ``involve'' four different cities. \qed
\end{example}

Let's follow the working principles of GAPX using the example in Figure~\ref{fig:gapx-tsp}. GAPX creates a combined graph with all the arcs of the two parent solutions. Next, it removes the arcs that are common to both parents (gray in the figure). Then, it finds the connected components in the graph. In one connected component, there are entry nodes and exit nodes. The arcs in the component connect each entry node to one exit node. If the entry-to-exit connections are the same in both parent solutions, then the decision of which parent to use for the arcs of that component can be taken independently of the other components. 

\begin{figure}[!th]
\includegraphics[scale=0.4]{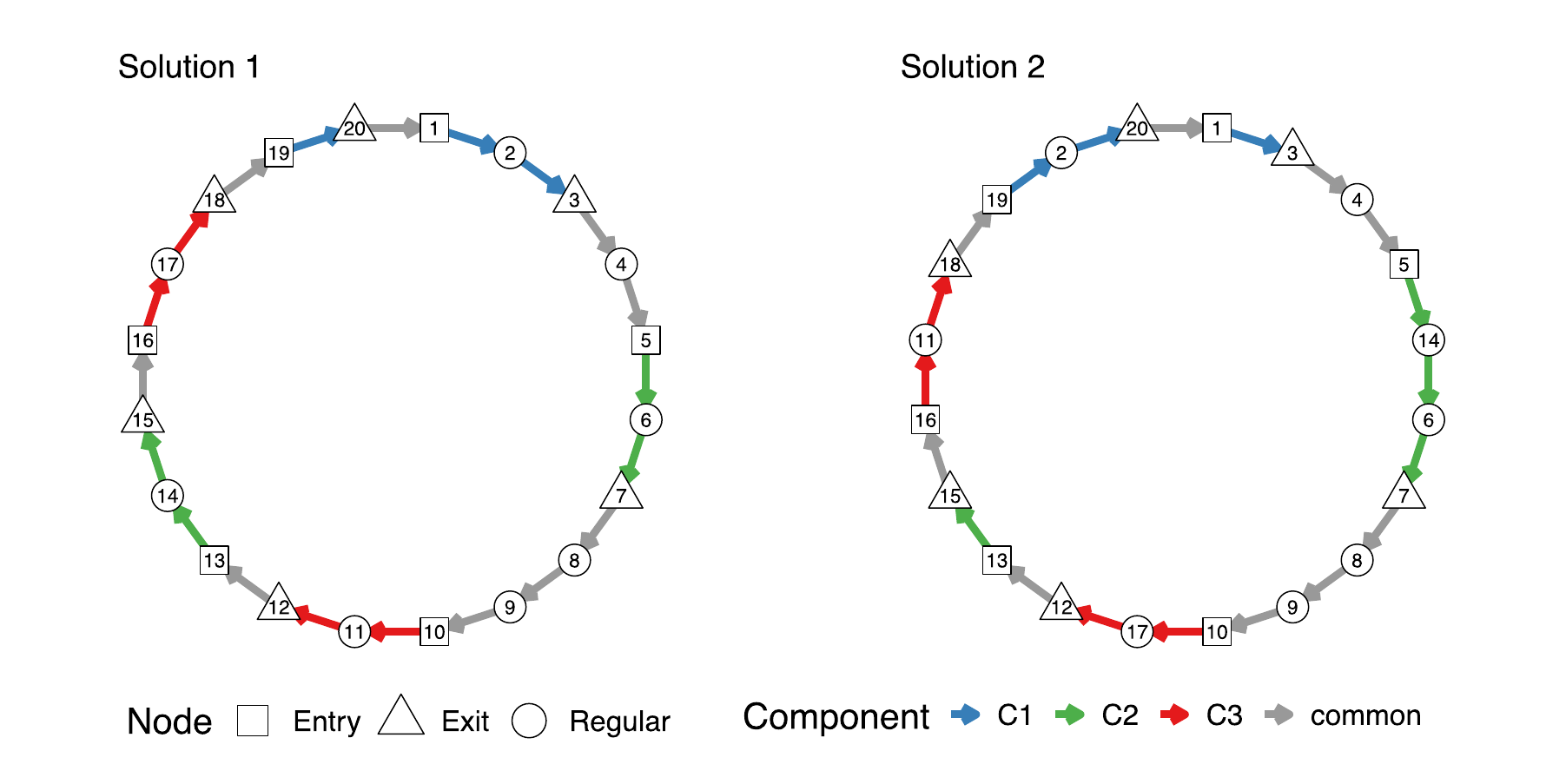}
\caption{Two parent solutions of GAPX in TSP.}
\label{fig:gapx-tsp}
\end{figure}

\begin{example}
In Figure~\ref{fig:gapx-tsp}, there are three components and the arcs in the different components are given different colors (red, blue and green). The entry nodes of the red component are 10 and 16. The exit nodes of the red component are 12 and 18. Observe that in both parents, 10 is connected to 12 in the component and 16 is connected to 18. Thus, in Solution 1 we could replace the arcs in red with the arcs in red in Solution 2 and we still have a valid solution (potential offspring). We can do this by applying insertion moves $h_{17,12}$ and $h_{11,18}$. These two moves commute but they are \emph{interacting}. 
We can easily see this in Table~\ref{tab:moves-solution1}, where the arcs $(16,11)$ and $(10,17)$ are only added by move $h_{17,12} \cdot h_{11,18}$. \qed
\end{example}
For two insertion moves to be non-interacting in a solution they must involve cities which are \emph{not} adjacent. In general, the moves that transform the paths in one component of Solution 1 to the paths in the same component of Solution~2 are non-interacting for Solution 1, even if we also apply the moves that transform the other components. Thus, Corollary~\ref{cor:lattice} is applicable.    
\begin{example}
Moves $h_{17,12}$, $h_{11,18}$, and $h_{17,12} \cdot h_{11,18}$ are non-interacting with move $h_{2,20}$, which transforms the blue component from Solution 1 to Solution 2. \qed
\end{example}

\begin{table}[!ht]
\centering
\caption{Arcs added and removed after applying different moves to Solution 1.}
\label{tab:moves-solution1}
\begin{tabular}{lllll}
\toprule
\multicolumn{1}{l}{\textbf{Moves}} & & \multicolumn{1}{l}{\textbf{Removed arcs}} & & \multicolumn{1}{l}{\textbf{Added arcs}} \\
\hline
$h_{11,18}$ & & $(10,11)$, $(11,12)$, $(17,18)$ & & $(10,12)$, $(17,11)$, $(11,18)$ \\
$h_{17,12}$ & & $(16,17)$, $(17,18)$, $(11,12)$ & & $(11,17)$, $(17,12)$, $(16,18)$ \\
$h_{17,12} \cdot h_{11,18}$ & & $(16,17)$, $(17,18)$, $(10,11)$, $(11,12)$ & & $(16,11)$, $(11,18)$, $(10,17)$, $(17,12)$ \\
\bottomrule
\end{tabular}
\end{table}

The previous argument shows that the way in which GAPX works produces a set of moves that are pairwise non-interacting. Thus, if GAPX finds $q$ components in this way, Corollary~\ref{cor:lattice} can be applied and the best of $2^q$ potential offspring is computed.
\begin{example}
In Figure~\ref{fig:gapx-tsp}, the pairwise non-interacting moves are $h_{17,12} \cdot h_{11,18}$ (red component), $h_{2,20}$ (blue component), and $h_{14,15}$ (green component). \qed
\end{example}

\subsection{Gray-box optimization in pseudo-Boolean functions}
\label{subsec:grabox-pseudo-boolean}

In the context of pseudo-Boolean optimization, we can represent solutions and moves with elements of $\mathbb{Z}_2^n$, which form a commutative group under addition. The gray-box operators defined in the literature are based on the concept of Variable Interaction Graph (VIG). A VIG is an undirected graph where the set of nodes is the set of Boolean variables and there is an edge between two variables when they interact.
Using the language of our framework, the nodes of the graph are the moves of size one (bit flip) and the edges represent interactions among those moves. Let us define $ones(g)=\{i \,|\, g_i=1, \text{for } 1 \leq i \leq n\}$. In terms of the Fourier (Walsh) transform, there is an interaction among two order-1 moves $g,h \in \mathbb{Z}_2^n$ when there is a nonzero Walsh coefficient $\hat{f}(\lambda)$ with $ones(\lambda) \cap ones(g) \neq \emptyset$ and $ones(\lambda) \cap ones(h) \neq \emptyset$. This means that the $\lambda$ vector is 1 in the positions where moves $g$ and $h$ are 1.
For higher-order moves $g, h \in \mathbb{Z}_2^n$, previous work says that they are non-interacting when the subgraphs induced by the moves in the VIG are not adjacent. That is, two moves $g,h \in \mathbb{Z}_2^n$ are non-interacting if for all $\lambda \in \mathbb{Z}_2^n$ at least one of the following three conditions is true: 1) $\hat{f}(\lambda)=0$, 2) $ones(\lambda) \cap ones(g) = \emptyset$, or 3) $ones(\lambda) \cap ones(h) = \emptyset$. If we denote with $\lambda g$ the dot product of $\lambda$ and $g$ in $(\mathbb{Z}_2^n,+,\cdot)$, we can observe that $\lambda g = 0$ when $ones(\lambda) \cap ones(g) = \emptyset$; and, thus, if one of the three previous conditions is true for all $\lambda$, then Eq.~\eqref{eqn:decomposition-fourier} is satisfied. As a consequence, moves $h$ and $g$ are non-interacting, and all the previous results on gray-box operators for pseudo-Boolean optimization (including hill climbers and partition crossover) are explained in the context of our new mathematical framework.

But our framework has more to say. There are cases in which $ones(\lambda) \cap ones(g) \neq \emptyset$ and still $\lambda g=0$. In particular, this happens exactly when $|ones(\lambda) \cap ones(g)|$ is even, because the sum in the dot product is the sum in $\mathbb{Z}_2$.

\begin{example}
\label{example:extra-decomposition}
Let us assume that we have a function $f$ that depends on three variables: $x_1$, $x_2$ and $x_3$, and let us assume that the Walsh transform provides nonzero values for $\hat{f}(100)=w_1$, $\hat{f}(010)=w_2$ and $\hat{f}(001)=w_3$, $\hat{f}(110)=w_{1,2}$ and $\hat{f}(111)=w_{1,2,3}$. Now we wonder if we can decompose move $h=111$ (all bits changing). A decomposition based on the VIG would not decompose this move. Thus, a Hamming Ball Hill Climber~\cite{ChicanoWS14} with radius 3 would need to store the score for move $111$ and a partition crossover would only find one component if the two parents differ in the three bits. This is due to the nonzero value for $\hat{f}(111)$.
However, a detailed analysis based on the Fourier (Walsh) transform reveals that this move can be decomposed into two moves, $h_1=110$ and $h_2=001$, as if $w_{1,2,3}$ were zero. The reason is that for move $h_1=110$ we have $\varphi_{111}(h_1)=(-1)^{111 \cdot 110}=1$, $\varphi_{110}(h_1)=1$ and $\varphi_{001}(h_1)=1$; and for move $h_2=001$ we have $\varphi_{100}=1$ and $\varphi_{010}=1$. These are the only nonzero Walsh coefficients and Eq.~\eqref{eqn:decomposition-fourier} is satisfied.

We can interpret this decomposition in another way. According to the inverse Fourier transform, Eq.~\eqref{eqn:inverse-fourier}, we can write $f$ in terms of the Fourier coefficients as
\begin{align*}
    f(x) = \frac{1}{8} &\left(
    w_1 \varphi_{100}(x) +
    w_2 \varphi_{010}(x) +
    w_3 \varphi_{001}(x) +
    w_{1,2} \varphi_{110}(x) +
    w_{1,2,3} \varphi_{111}(x)
    \right) .
\end{align*}

Move $h_1=110$ will only affect the coefficients $w_1$ and $w_2$ in the Fourier transform (it changes sign of $\varphi_{100}(x)$ and $\varphi_{010}(x)$), while move $h_2=001$ will only affect the coefficients $w_{1,2,3}$ and $w_{3}$. Since the coefficients affected by the moves are different, then there is no interaction between the moves.

However, this decomposition does not mean we can ignore the Fourier coefficient $w_{1,2,3}$ for
all possible purposes. For example, let us assume the input for partition crossover are the two solutions $x_1 =011$ and $x_2= 000$. The move $011$ cannot be decomposed in $001$ and $010$ because these moves both affect the sign of $w_{1,2,3}$. 
Thus, the decomposition opportunities depend on the concrete parent solutions. Partition crossover can find the decomposition opportunities dynamically or we can also statically pre-compute some opportunities depending on the values of some particular variables to speed up the process.
\qed
\end{example}

\section{New gray-box operators from the unifying framework}
\label{sec:new-gb-operators}

We finish this work by illustrating how the new mathematical framework for gray-box optimization opens the door to the design of new gray-box operators. In particular, we focus on two NP-hard permutation problems: the Linear Ordering Problem (LOP)~\cite{MartiReinelt2011} and Single Machine Total Weighted Tardiness Problem (SMTWTP)~\cite{LENSTRA1977}. 
For LOP the goal is to reorder the columns and rows of a matrix to minimize the sum of the upper diagonal. The  objective function is
\begin{equation}
    \label{eqn:lop}
    f_{\rm LOP}(\sigma) = \sum_{i=1}^n \sum_{j=i+1}^n A_{\sigma(i),\sigma(j)},
\end{equation}
where  $A$ is a $n \times n$ real matrix. In SMTWTP the goal is to find an order for $n$ jobs minimizing the weighted tardiness. The objective function is
\begin{equation}
    \label{eqn:smtwtp}
    f_{\rm SMTWTP}(\sigma) = \sum_{i=1}^{n} w_{\sigma(i)} \cdot T_{\sigma(i)},
\end{equation}
where $w_j$ and $T_{j}$ is the weight and tardiness of the $j$-th job, respectively. Each job $j$ also has a due time $d_j$ and a processing time $t_j$. The tardiness for job $j$ is $T_j = \max \{0, C_j - d_j\}$, where $C_j = \sum_{i=1}^{\sigma^{-1}(j)} t_{\sigma(i)}$ is the completion time for job $j$, which is the sum of the processing time of job $j$ and the preceding jobs.

It is easy to see that in both problems re-arranging sets of consecutive elements in the permutation does not affect the contributions to the fitness function of the elements below or above the set. Given a solution $\sigma$ for LOP, if we apply a move that arbitrarily permutes the elements in positions $i$ to $j$, the term $A_{\sigma(l), \sigma(k)}$ for $i\leq l \leq j$ and $k > j$ will still be part of the value of the objective function for the perturbed solution. The same happens for the terms $A_{\sigma(k), \sigma(l)}$ with $i\leq l \leq j$ and $k < i$. In SMTWTP, it is also easy to see that the tardiness of jobs in positions above $j$ or below $i$ is not affected by an arbitrary permutation of the elements between positions $i$ and $j$ and, thus, their contribution to the weighted tardiness (objective function) is unaffected.

This property allows us to confirm that Eq.~\eqref{eqn:delta-sum} is fulfilled when $h_1$ and $h_2$ are permutations affecting disjoint consecutive elements. More precisely, if $e(h_1)$ is the set of elements in the permutation changed by $h_1$ and $e(h_2)$ is the set of elements changed by $h_2$, then we should have $[\min(e(h_1)),\max(e(h_1))] \cap [\min(e(h_2)),\max(e(h_2))] = \emptyset$ for Eq.~\eqref{eqn:delta-sum} to be true in LOP and SMTWTP.

Based on this decomposition, we can immediately design an efficient hill climber and a partition crossover. Regarding the hill climber, we only need to store the values $\Delta_{h} f$ for moves $h$ that permute adjacent elements in the permutation. 
\begin{example}
We could consider the case in which $h$ are swaps (exchange of adjacent positions). There are $n-1$ possible swaps in a permutation. We could also consider the case of any arbitrary permutation for three consecutive elements. There are $n-2$ sets of three consecutive elements and five non-trivial permutations of three elements, so we have $5(n-2)$ values to store.
\qed
\end{example}
In general, this approach will require $O(n)$ memory. If a move is taken during the hill climbing, then updating the $\Delta_{h} f$ values has to be done only for \emph{overlapping} moves, which are a constant number, and can be done in constant time.

Regarding the construction of a partition crossover operator, given two parent solutions $\sigma_1$ and $\sigma_2$, the idea is to compute the decomposition for $\sigma_2 \cdot \sigma_1^{-1}$. We can find a decomposition by finding permutations of consecutive elements in $\sigma_2 \cdot \sigma_1^{-1}$. The pseudo-code for this is in Algorithm~\ref{alg:lop-px}.

\begin{algorithm}
\caption{Partition crossover for LOP/SMTWTP}
\label{alg:lop-px}
\begin{algorithmic}[1]
\STATE $\sigma_* \leftarrow \sigma_1$
\STATE $i \leftarrow 1$
\WHILE{$i \leq n$}
\STATE $h \leftarrow (1)$ // Initialize with identity permutation
\STATE $l \leftarrow i$
\STATE $h(l) \leftarrow \sigma_2(\sigma_1^{-1}(l))$
\STATE $j \leftarrow h(l)$
\WHILE{$l < j$}
\STATE $l \leftarrow l+1$
\STATE $h(l) \leftarrow \sigma_2(\sigma_1^{-1}(l))$
\STATE $j \leftarrow \max(j,h(l))$
\ENDWHILE { // Component $h$ found between positions $i$ and $j$}
\IF{$\left(\Delta_h f\right)(\sigma_1) < 0$}
\STATE $\sigma_* \leftarrow h(\sigma_*)$
\ENDIF
\STATE $i \leftarrow i+1$
\ENDWHILE
\RETURN $\sigma_{*}$
\end{algorithmic}
\end{algorithm}

The operators here defined for LOP and SMTWTP can also be applied with a slight variation to the TSP. First, we need to consider that the TSP has a rotational symmetry and, thus, we need to consider that element $n$ is adjacent to element $1$. And second, in TSP if $h_1$ touches element $i$ and $h_2$ touches element $i+1$, then $h_1$ and $h_2$ will interact (they both affect edge $(\sigma(i),\sigma(i+1))$). Thus, we need to ensure that this does not happen. It is not enough that $e(h_1) \cap e(h_2) = \emptyset$, we also need that no element in $e(h_1)$ is adjacent to an element in $e(h_2)$.

\section{Conclusions}
\label{sec:conclusions}

This paper proposes a mathematical framework that unifies two gray-box operators defined in the literature: Hamming Ball Hill Climber (HBHC) and Partition Crossover (PX). The framework also generalizes the results in the literature and provides a mathematical background to define particular versions of these operators in search spaces different from the ones used in the past, which are permutations and binary strings. 
With the help of the Fourier transform over finite groups we can identify opportunities to design more effective operators.


This paper opens a new research line in gray-box optimization, allowing faster development of new gray-box operators and a new re-design of the old ones to make them more effective. There are some open questions that can be explored in future work, like the extension of the framework to sets of non-commutative moves or the research for the most appropriate representation of a problem to get profit from the efficient gray-box operators.
In practice, it would be interesting to re-design the operators HBHC and PX for pseudo-Boolean optimization taking into account the new opportunities for decomposition revealed by Theorem~\ref{thm:decomposition-fourier}. New gray-box operators can be designed for combinatorial optimization problems for which it was not clear how to apply the existing gray-box operators.

\begin{credits}
\subsubsection{\ackname} 

This research is partially funded by project PID 2020-116727RB-I00 (HUmove) funded by MCIN/AEI/ 10.13039/501100011033; TAILOR ICT-48 Network (No 952215) funded by EU Horizon 2020 research and innovation programme; Junta de Andalucia, Spain, under contract QUAL21 010UMA; and the University of Malaga (PAR 4/2023). This work is also partially funded by a National Science Foundation (NSF) grant to D. Whitley, Award Number: 1908866.

\subsubsection{\discintname}
The authors have no competing interests to declare that are
relevant to the content of this article.
\end{credits}

%
%
%
\bibliographystyle{splncs04}
\bibliography{biblio,genetic}
\end{document}